\documentclass[sigconf]{acmart}
\usepackage{booktabs}
\usepackage{multirow}
\usepackage{siunitx} 
\usepackage{amsmath}
\usepackage{multirow}
\usepackage{algorithm}
\usepackage{algorithmicx}
\usepackage{algpseudocode}
\usepackage{tabularx}
\usepackage{listings}
\usepackage{hyperref}
\AtBeginDocument{%
  }

\setcopyright{acmlicensed}
\copyrightyear{2025}
\acmYear{2025}
\acmConference[MM '25] {Proceedings of the 33rd ACM International Conference on Multimedia}{October 27--31, 2025}{Dublin, Ireland.}
\acmBooktitle{Proceedings of the 33rd ACM International Conference on Multimedia (MM '25), October 27--31, 2025, Dublin, Ireland}

\acmISBN{979-8-4007-2035-2/2025/10}
\acmDOI{10.1145/3746027.3755441}

\begin{document}

\title{Text2Weight: Bridging Natural Language and Neural Network Weight Spaces}

\author{Bowen Tian}
\authornote{Both authors contributed equally to this research.}
\orcid{0009-0000-1305-3097}
\affiliation{%
\institution{Deep Interdisciplinary Intelligence Lab}
  \institution{The Hong Kong University of Science and Technology (Guangzhou)}
  \city{Guangzhou}
  \country{China}
}
\email{bowentian@hkust-gz.edu.cn}

\author{Wenshuo Chen\footnotemark[1]}
\orcid{0009-0002-1966-6059}
\affiliation{%
  \institution{Deep Interdisciplinary Intelligence Lab}
  \institution{The Hong Kong University of Science and Technology (Guangzhou)}
  \city{Guangzhou}
  \country{China}
}
\email{wenshuochen@hkust-gz.edu.cn}

\author{Zexi Li}
\orcid{0000-0003-0831-3549}
\affiliation{%
\institution{The University of Cambridge}
\state{Cambridge}
\country{UK} \\
\institution{Zhejiang University}
\city{Hangzhou}
\country{China}
}
\email{zexi.li@zju.edu.cn}

\author{Songning Lai}
\orcid{0009-0007-3132-9414}
\affiliation{%
\institution{Deep Interdisciplinary Intelligence Lab}
  \institution{The Hong Kong University of Science and Technology (Guangzhou)}
  \city{Guangzhou}
  \country{China}
}
\email{songninglai@hkust-gz.edu.cn}

\author{Jiemin Wu}
\orcid{0009-0005-2712-4876}
\affiliation{%
\institution{Deep Interdisciplinary Intelligence Lab}
  \institution{The Hong Kong University of Science and Technology (Guangzhou)}
  \city{Guangzhou}
  \country{China}
}
\email{jwu663@connect.hkust-gz.edu.cn}

\author{Yutao Yue}
\authornote{Correspondence to Yutao Yue \{yutaoyue@hkust-gz.edu.cn\}}
\orcid{0009-0005-2712-4876}
\affiliation{%
\institution{Thrust of Artificial Intelligence and Thrust of Intelligent Transportation}
  \institution{The Hong Kong University of Science and Technology (Guangzhou)}
  \city{Guangzhou}
  \country{China} \\
\institution{Institute of Deep Perception Technology}
\institution{Jiangsu Industrial Technology Research Institute}
\city{Wuxi}
\country{China}
}
\email{yutaoyue@hkust-gz.edu.cn}

\begin{abstract}
How far are we really from automatically generating neural networks? While neural network weight generation shows promise, current approaches struggle with generalization to unseen tasks and practical application exploration. To address this, we propose \textbf{T2W}, a diffusion transformer framework that generates task-specific weights conditioned on natural language descriptions. T2W hierarchically processes network parameters into uniform blocks, integrates text embeddings from CLIP via a prior attention mechanism, and employs adversarial training with weight-space augmentation to enhance generalization. Experiments on Cifar100, Caltech256, and TinyImageNet demonstrate T2W’s ability to produce high-quality weights for unseen tasks, outperforming optimization-based initialization and enabling novel applications such as weight enhancement and text-guided model fusion. Our work bridges textual semantics with weight-space dynamics, supported by an open-source dataset of text-weight pairs, advancing the practicality of generative models in neural network parameter synthesis. Our code is available on  \href{https://github.com/TianSuya/T2W}{Github}.
\end{abstract}


%
\begin{teaserfigure}
\vspace{-10pt}
  \centering
  \includegraphics[width=0.8\textwidth]{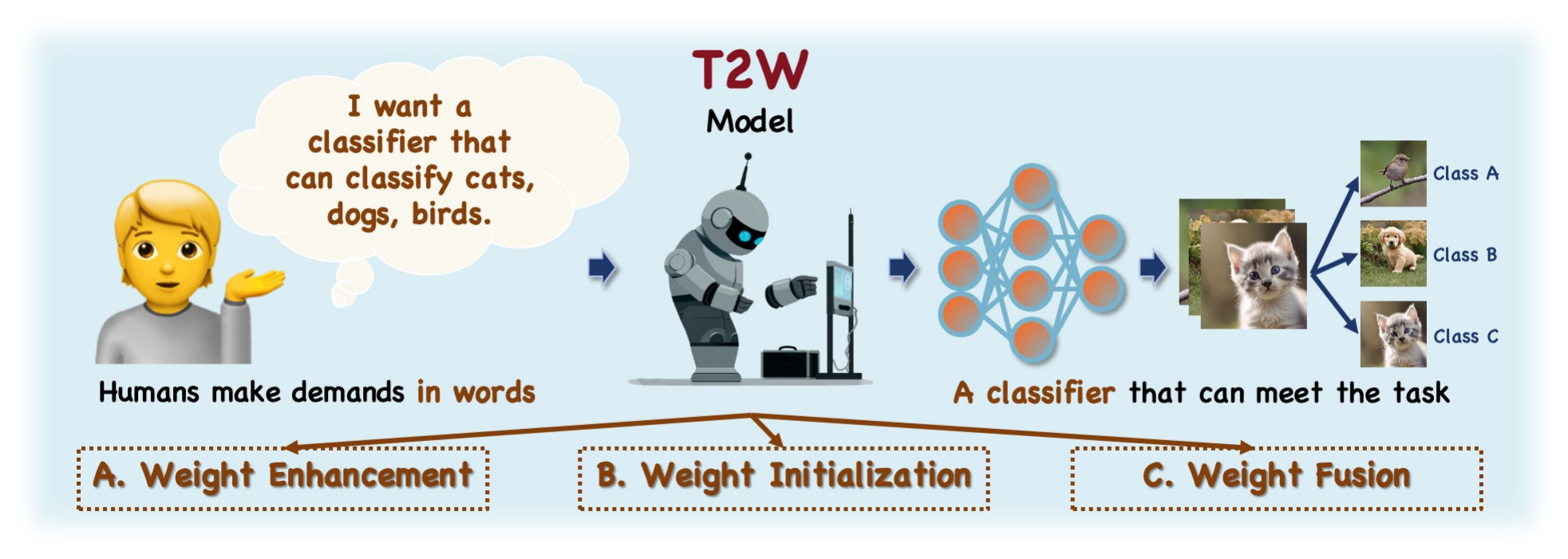}
  \vspace{-8pt}
  \caption{We can construct links between the \textbf{text space} and the \textbf{deep model weight space}, upon receiving a request for a classifier function given in the form of text, our T2W model can generate the appropriate classifier weights for use. And this technique has strong potential applications. Such as \underline{weight enhancement}, \underline{weight initialization}, and \underline{weight fusion}.}
  \label{fig:teaser}
\vspace{+4pt}
\end{teaserfigure}

\maketitle

\vspace{-8pt}
\section{Introduction}
Deep learning advances have enabled neural networks to process diverse modalities, including images \cite{voulodimos2018deep, tian2025pepl, he2016deep}, text \cite{vaswani2017attention, wang2025mdpo}, and audio \cite{kreuk2022audiogen, borsos2023audiolm}. \textbf{Recent research explores neural network weights as a novel data modality}. Studies \cite{lim2023graph, navon2023equivariant, zhou2023permutation} focus on designing meta-networks by exploiting network symmetries, while \cite{schurholt2021self, schurholt2022hyper} develop hyper-representations through weight-space data augmentation. Works like \cite{ilharco2022editing, ainsworth2022git} investigate weight fusion properties.

The rapid progress in generative AI has enabled high-quality generation across diverse data modalities. In the domain of neural network weight generation, diffusion-based methods have attracted significant attention due to their robust modeling, generalization, and stability. For instance, \cite{peebles2022learning} pioneered the use of diffusion models for conditional parameter generation by leveraging model loss as a condition. Subsequent works, including \cite{jin2024conditional, wu2024difflora, gong2024efficient, nava2022meta}, and \cite{xie2024weight}, have expanded diffusion-based weight generation to broader scenarios. However, most existing methods rely on generating weights trained on seen tasks, limiting generalization to \textbf{unseen tasks (tasks not seen during the training phase)} and hindering practical utility \cite{jin2024conditional, wu2024difflora}. We attribute this limitation to the narrow distribution of generation conditions (e.g., loss values \cite{peebles2022learning} or dataset embeddings \cite{jin2024conditional}), which fail to guide models in capturing out-of-distribution tasks. Thus, developing a more generalizable and versatile weight generation framework is critical for advancing this field.

In this paper, we propose \textbf{T2W} (\underline{\textbf{T}}ext\underline{\textbf{2}}\underline{\textbf{W}}eight), a diffusion transformer (DiT) based model\cite{peebles2023scalable} that generates neural network weights conditioned on textual task descriptions. The DiT architecture is particularly advantageous for neural parameter synthesis: its self-attention mechanism inherently handles hierarchical parameter blocks as sequential tokens, enabling global dependency modeling across distributed weight components—capabilities limited in traditional DDPMs that rely on local convolutional priors. Unlike latent diffusion approaches \cite{rombach2022high} that compress data into lower-dimensional spaces, T2W operates directly on the structured weight space for valid weight generation. Specifically, we partition network weights into uniform blocks after hierarchical grouping, forming a sequence of parameter blocks. Task descriptions (detailed in \autoref{subsec:dataset}) are embedded into text feature vectors via CLIP’s text encoder \cite{radford2021learning} and fused with parameter blocks through a prior attention mechanism (detailed in \autoref{sec:experiments}) to construct T2W’s input. Before entering the transformer, each input component is mapped to fixed-dimensional tokens. The transformer’s output predicts the denoised neural network parameters.

During training, we enhance generation quality by incorporating weight-space permutation symmetry constraints and introducing a weight-space discriminator for adversarial training. We construct extensive weight datasets on CIFAR-100 \cite{krizhevsky2009learning}, Caltech-256 \cite{griffin2007caltech}, and TinyImageNet \cite{tiny-imagenet} (\autoref{subsec:dataset}) to train and evaluate T2W’s effectiveness, achieving compelling experimental results.

Our experiments demonstrate that T2W successfully generalizes to unseen tasks, generating high-quality weights. The average classification accuracy of the classifiers generated for the TinyImageNet sub-dataset is over \textbf{80\%} in unseen tasks. T2W producing weights exhibiting superior generalization and robustness (detailed in \autoref{fig:loss}). We validate T2W’s capabilities in three downstream applications:  
\textbf{(1) Weight initialization}: T2W generated weights outperform random initialization in supervised learning.  
\textbf{(2) Post-training weight enhancement}: Denoising pre-trained weights via T2W yields improved performance.  
\textbf{(3) Weight fusion}: By translating weight fusion into text-description fusion, T2W achieves promising results compared to conventional methods.

Our contributions are summarized as follows: \underline{\textbf{(i)}} We propose T2W, bridging text features to model weights and advancing multimodal understanding of weight-space as a distinct modality. \underline{\textbf{(ii)}} We curate and open-source a large-scale text-weight paired dataset to foster progress in weight generation. \underline{\textbf{(iii)}} We explore diverse downstream applications of T2W (e.g., weight initialization, enhancement, fusion), demonstrating its practical potential and value.

\vspace{-8pt}
\section{Problem Formulation}
We aim to synthesize neural network weights directly from natural language task descriptions, bridging textual semantics and functional parameterizations. First we can give detailed definitions of \underline{Textual Task Space} and \underline{Weight Space}:

\noindent \textbf{Textual Task Space. } $\mathcal{C}$ be the set of natural language task descriptions $c$, where $c$ is a detailed task description list (e.g.,["a photo of bird", "a photo of a cat",...]).

\noindent \textbf{Weight Space. } $\Theta \in \mathbb{R}^d$ denotes the space of valid neural network weights, constrained by architectural symmetry and trainability (e.g., permutation invariance across hidden units).

The core problem is to learn a conditional generative distribution $p(\theta_g|c)$, where $\theta_g \in \Theta$, that satisfies three fundamental properties:

\noindent\textbf{A. Task-Weight Consistency}: Generated weights must induce a model $f_{\theta_{g}}$ whose behavior aligns with the semantics of $c$. Formally, for a task-specific performance metric $\mathcal{M}(f_{\theta_{g}},c)$ we require:
\begin{equation}
    \mathbb{E}_{\theta_g \sim p(\theta_g|c)}[\mathcal{M}(f_{\theta_{g}},c)] \geq \tau,
\end{equation}
\noindent where $\tau$ is a task-dependent threshold. For classification tasks, $\mathcal{M}(\cdot)$ could represent accuracy on a validation set.

\noindent\textbf{B. Weight-Space Validity}: The generated $\theta_g$ must reside on the manifold $\mathcal{M}_\Theta$ of trainable parameters. This requires $\theta_g$ to satisfy architectural symmetry constraints. Let $\mathcal{G}$ be the group of symmetry transformations (e.g., permutations of neurons in a layer). Validity is enforced by:
\begin{equation}
    p(g\cdot\theta_g|c) = p(\theta_g|c)\ \ \ \ \forall g\in \mathcal{G},\ \forall \theta_g \in \Theta,
\end{equation}
\noindent ensuring invariance to symmetry operations during generation. However, in actual practice, existing generative models often have difficulty with this constraint, which needs to be addressed in subsequent work.

\noindent\textbf{C. Generalization to Unseen Tasks:} The model must generalize to task descriptions $c' \in \mathcal{C}_{unseen}$, where $\mathcal{C}_{unseen}$ is disjoint from the training distribution $\mathcal{C}_{seen}$. Let $\mathcal{D}_{train} = \{(\theta^i,c^i)|i\in[1,N_{train})\}, c\in\mathcal{C}_{seen}$ be the training set and $\mathcal{D}_{test} = \{(\theta^i,c'^i)|i\in[1,N_{test})\}, c'\in\mathcal{C}_{unseen}$ be the test set, We require:
\begin{equation}
    \mathbb{E}_{c'\sim\mathcal{C}_{unseen}}\big[\mathbb{E}_{\theta_g \sim p(\theta_g|c')}[\mathcal{M}(f_{\theta_{g}},c')]\big] \geq \tau_{test},
\end{equation}
\noindent where $\tau_{test}$ defines acceptable out-of-distribution performance.

\section{T2W: Our Method}
\label{sec:method}
We present the dataset construction and preprocessing methods in \autoref{subsec:dataprep}, the loss design for the Diffusion training process in \autoref{subsec:diffusionloss}, the loss design for the symmetry alignment constraints in the weight space in \autoref{subsec:permutationloss}, along with the loss design for the adversarial constraints in \autoref{subsec:adversarialloss}, and finally the overall loss function composition in \autoref{subsec:overallloss}. \autoref{fig:overview} illustrates the overview of our method.

\begin{figure*}[htbp]
\centering
\includegraphics[width=0.75\linewidth]{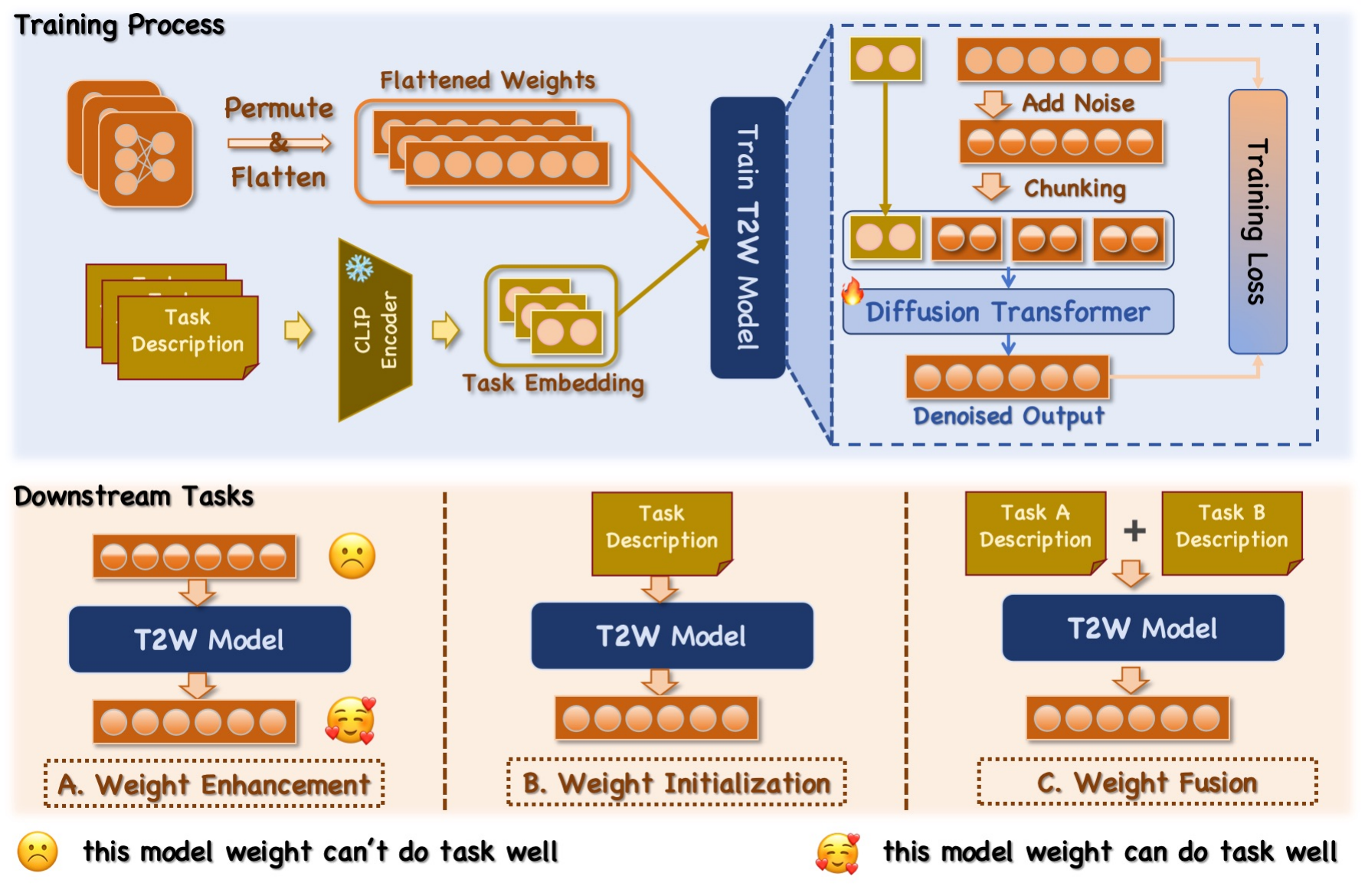}
\vspace{-5pt}
\caption{An overview of the \textbf{T2W} framework.}
\label{fig:overview}
\vspace{-5pt}
\end{figure*}

\subsection{Data Preparation}
\label{subsec:dataprep}
Since the neural network weight space $\mathbb{R}^d$ is often a high-dimensional space, if the weights are directly flatten as one-dimensional vectors for model inputs, the subsequent linear layers are often mapped to a relatively low feature dimension resulting in the loss of key information of the model weights under the excessive dimensionality reduction.

Inspired by \cite{peebles2022learning, li2024text}, we perform an initial chunking decoupling before the model weights are fed into the Transformer module. This limits each block of input to a certain size, and then passes through the mapping layer without losing too much information, which can be expressed as:
\begin{equation}
    chunk(\theta_{in})=[\theta_{in}^1,\theta_{in}^2,\dots,\theta_{in}^{N_c}],
\end{equation}
\noindent where $\theta_{in} \in \mathbb{R}^d$, $\theta_{in}^i \in \mathbb{R}^{\lfloor\frac{d}{s}\rfloor},\ i\in [1, N_c)$, $\theta_{in}^{N_c} \in \mathbb{R}^{d-\lfloor\frac{d}{s}\rfloor\cdot(N_c-1)}$. $s$ denotes the size of each block after chunking, and $N_c$ denotes the number of blocks after chunking.

After chunking, we map the data separately (including the text condition vector $v_c$, which is detailed in \autoref{subsec:diffusionloss}) according to the configuration before it enters the Transformer module, so that they are all converted to a uniform dimension that is entered into the Transformer as a Token. Which can be expressed as:
\begin{equation}
    {\tilde{\theta}}^i_{in} = \pi^i(\theta_{in}^i),\ i\in [1,N_c],\ {\tilde{\theta}}_{in}^i \in \mathbb{R}^h,
\end{equation}
\noindent where $h$ denotes the hidden size of each token and $\pi$ denotes the projection transformation.

\subsection{Diffusion Loss Design}
\label{subsec:diffusionloss}
\noindent\textbf{Forward Process}: Given a neural network weight vector $\theta_0$, we define the forward diffusion process as a Markov chain gradually adding Gaussian noise over N steps:
\begin{equation}
\begin{aligned}
    q\left(\theta_{1: N} \mid \theta_{0}\right)&=\prod_{n=1}^{N} q\left(\theta_{n} \mid \theta_{n-1}\right),\\
    \quad q\left(\theta_{n} \mid \theta_{n-1}\right)=\mathcal{N}&\left(\theta_{n} ; \sqrt{1-\beta_{n}} \theta_{n-1}, \beta_{n} \mathbf{I}\right),
\end{aligned}
\end{equation}
\noindent where $\beta_n \in(0,1)$ is a predefined noise schedule.

\noindent \textbf{Reverse Process}: The denoising network $\epsilon_{\phi}(\cdot)$ learns to reverse the diffusion process, conditioned on textual task descriptions $c$. At each step $n$, it predicts the noise added to $\theta_n$:
\begin{equation}
    p_{\phi}(\theta_{n-1}|\theta_n, v_c) = \mathcal{N}\big(\theta_{n-1};\mu_{\phi}(\theta_n,n,v_c),\Sigma_n\big),
\end{equation}
\noindent where $v_c = E_{clip}(c) \in \mathbb{R}^{d_c}$, $E_{clip}(\cdot)$ denotes the process of encoding the text description based on the clip text encoder, resulting in a vector containing the semantics of the text ($v_c$). $d_c$ denotes the dimension of the textual semantic vector. and $\mu_\phi$ is parameterized as:
\begin{equation}
    \mu_\phi(\theta_n,n,v_c) = \frac{1}{\sqrt{\alpha_n}}\left(\theta_n - \frac{\beta_n}{\sqrt{1-\bar{\alpha}_n}}\epsilon_\phi(\theta_n,n,v_c)\right),
\end{equation}
\noindent here, $\alpha_n=1-\beta_n$, $\bar{\alpha}_n=\Pi_{i=1}^n\alpha_i$, and $\Sigma_n=\frac{1-\bar{\alpha}_{n-1}}{1-\bar{\alpha}_n}\beta_n\mathbf{I}$.

\noindent \textbf{Training Objective}: The diffusion loss minimizes the MSE between forward and reverse processes:
\begin{equation}
    \mathcal{L}_{\text{diff}} = \mathbb{E}_{n,\theta_0,\epsilon}\left[ \parallel \epsilon-\epsilon_\phi\left( \sqrt{\bar{\alpha}_n}\theta_0 + \sqrt{1-\bar{\alpha}_n}\epsilon,n,v_c \right) \parallel^2 \right],
\end{equation}
\noindent where $\epsilon \sim \mathcal{N}(0,\mathbf{I})$, $n \sim \mathcal{U}\{1,N\}$, and $\theta_0$ is a ground-truth weight vector.

\subsection{Permutation-Equivariant Constraint}
\label{subsec:permutationloss}
To enforce invariance to weight-space symmetries (e.g., neuron permutations in a layer \cite{ainsworth2022git}), we define a group $\mathcal{G}$ of valid transformations $g:\Theta \rightarrow \Theta$. For each training weight , we augment the data by sampling $g \sim \mathcal{G}$ and applying:
\begin{equation}
    \theta^{aug}_0 = g\cdot\theta_0.
\end{equation}
And the denoising network must satisfy equivariance:
\begin{equation}
    \epsilon_\phi(g\cdot\theta_n,n,v_c) = g\cdot \epsilon_\phi(\theta_n,n,v_c),
\end{equation}
\noindent Based on this observation, we can construct a symmetric constraint loss on the weight space, denoted as:
\begin{equation}
    \mathcal{L}_{\text{sym}} = \mathbb{E}_{g\sim\mathcal{G}}\left[\parallel \epsilon_\phi(g\cdot\theta_n,n,v_c) - g\cdot \epsilon_\phi(\theta_n,n,v_c) \parallel^2 \right]
\end{equation}
Much of the existing work relies on direct permutation data augmentation of the weight data to achieve this constraint, and in \autoref{app:theoretical} we demonstrate the superiority of direct explicit loss constraints for T2W, as well as the theoretical error of permutation data augmentation.


\begin{table*}[!ht]
\centering
\caption{Performance Comparison on Seen and Unseen Tasks. where the best results (or near-best results) are labeled in \textbf{bold} and the next best results are \underline{underlined}.}
\label{tab:performance}
\vspace{-8pt}
\normalsize
\resizebox{0.8\linewidth}{!}{
\begin{tabular}{@{}l *{4}{S[table-format=1.2e-1] S[table-format=2.2]}@{}}
\toprule
\multirow{3}{*}{Methods} & \multicolumn{8}{c}{Datasets} \\
\cmidrule(lr){2-9}
 & \multicolumn{2}{c}{CIFAR-100} & \multicolumn{2}{c}{Tiny-ImageNet} & \multicolumn{2}{c}{Caltech-256} & \multicolumn{2}{c}{Avg} \\
\cmidrule(lr){2-3} \cmidrule(lr){4-5} \cmidrule(lr){6-7} \cmidrule(lr){8-9}
 & {Loss} & {Accuracy} & {Loss} & {Accuracy} & {Loss} & {Accuracy} & {Loss} & {Accuracy} \\
\midrule
\multicolumn{9}{@{}l}{\textit{Seen Task}} \\
\midrule
Universal Model    & 5.44e-5 & 55.33 & 3.98e-4 & 63.72 & 2.45e-5 & 72.33 &   {-}    & 63.79 \\
T2W-CLIP         & 2.56e-5 & \textbf{65.67} & 1.27e-4 & \underline{69.24} & 2.24e-5 & \underline{81.53} &     {-}   & \underline{72.15} \\
T2W-NL           & 1.84e-5 & \underline{65.59} & 7.14e-5 & \textbf{70.31} & 1.67e-5 & \textbf{82.64} &     {-}   & \textbf{72.85} \\
\midrule
\multicolumn{9}{@{}l}{\textit{Unseen Task}} \\
\midrule
Universal Model    & 4.12e-5 & 56.74 & 4.13e-4 & 62.75 & 2.57e-5 & 72.59 &    {-}     & 64.03 \\
T2W-CLIP         & 2.98e-5 & \textbf{65.99} & 1.25e-4 & \underline{70.42} & 2.21e-5 & \underline{81.89} &      {-}  & \underline{72.77} \\
T2W-NL           & 2.35e-5 & \textbf{65.98} & 7.94e-5 & \textbf{71.51} & 1.69e-5 & \textbf{83.15} &     {-}   & \textbf{73.55} \\
\bottomrule
    \end{tabular}}
\end{table*}

\begin{figure*}[h]
\centering
\includegraphics[width=0.85\linewidth]{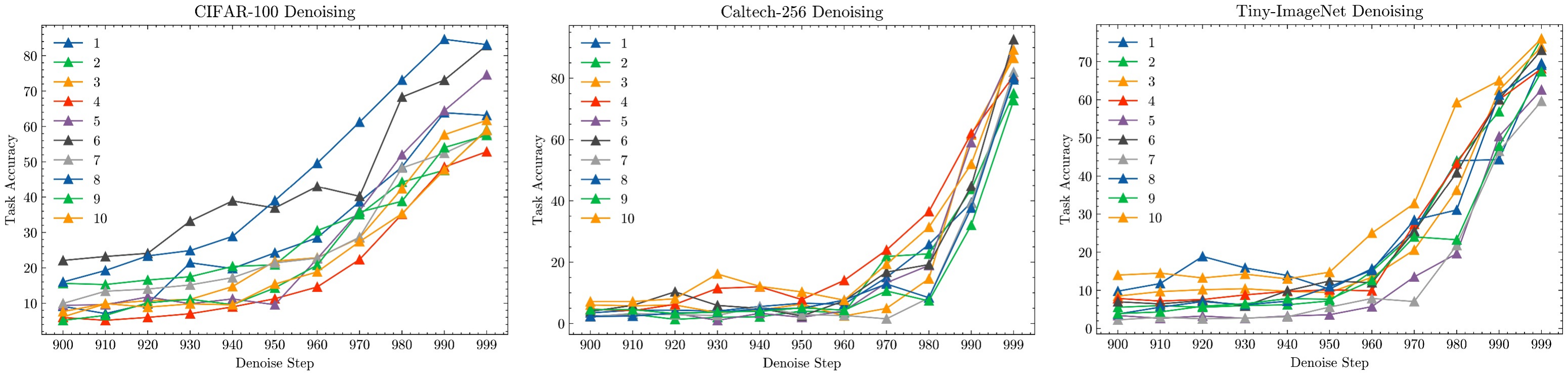}
\vspace{-5pt}
\caption{We have taken samples from the Unseen Task of each dataset for the visualization of denoising.}
\label{fig:denoise-vis}
\vspace{-0.5cm}
\end{figure*}

\subsection{Adversarial Training Objective}
\label{subsec:adversarialloss}
We assume that the weight space conforms to a latent data distribution pattern, while the denoising process has no a prior preconceptions about the distribution of the generated data since it is predicted from the noise, in order to place constraints on the distribution of the data generated by the diffusion process, we introduce an \textbf{adversarial training loss}, specifically, we train a discriminator for determining whether the source of a given weight is a generator or a real distribution. By improving the performance of the discriminator, we can give certain constraints on the data distribution generated by Diffusion.

A weight-space discriminator $D_\Phi:\Theta \rightarrow [0,1]$ is trained to distinguish real weights $\theta_0$ from generated weights $\theta_{gen} \sim p_\phi(\theta|c)$. The adversarial loss follows:
\begin{equation}
    \mathcal{L}_{\text{adv}} = \mathbb{E}_{\theta_0 \sim p_{data}}[\log D_{\Phi}(\theta_0)] + \mathbb{E}_{\theta_{gen}\sim p_\phi}[\log (1-D_\Phi(\theta_{gen}))]
\end{equation}
The discriminator $D_\Phi$ guides $\epsilon_\phi$ to produce weights that lie on the manifold of valid neural network weights.

\subsection{Overall Loss Objective}
\label{subsec:overallloss}
The total loss combines diffusion reconstruction, adversarial training, and symmetry regularization:
\begin{equation}
    \mathcal{L}_{\text{total}} = \mathcal{L}_{\text{diff}} + \lambda_1\cdot\mathcal{L}_{\text{sym}} + \lambda_2\cdot\mathcal{L}_{\text{adv}}
\end{equation}
where $\lambda_1$ and $\lambda_2$ are used as hyperparameters to control the magnitude of other losses.

\section{Experiments}
\label{sec:experiments}

\subsection{Dataset Construction and Training Details}
\label{subsec:dataset}
\textbf{Data Sources}:
We establish text-to-weight correspondence through a specialized vision-language architecture. Our framework employs:

\noindent \textbf{1. Frozen Feature Extractor}: Pre-trained ResNet-18 \cite{he2016deep} (weights fixed) for image feature extraction, producing $512$-dimensional embeddings

\noindent \textbf{2. Trainable CLIP-Projection Head}: A two-layer adapter mapping ResNet features to CLIP's text space:
    \begin{equation}
        W_{\text{head}} = \{W_1 \in \mathbf{R}^{512\times16}, W_2 \in \mathbf{R}^{16\times512}\}
    \end{equation}

\noindent \textbf{3. Task-Specific Generation Target}: Only $W_{\text{head}}$ parameters are synthesized, enabling architecture-agnostic classification through CLIP compatibility

\noindent For each base dataset (CIFAR-100, Caltech256, TinyImageNet), we generate 12,000 subtasks via:

\noindent \textbf{1. Random class selection:} $k \sim \mathcal{U}\{8,32\}$ classes per subtask.

\noindent \textbf{2. Text description templating:} ["A photo of {class\_1}", "A photo of {class\_2}", ..., "A photo of {class\_k}"].

\noindent \textbf{3. CLIP text encoding:} Since the textual description of a given k-categorization task is actually a list of strings, each representing a natural language description of a given category, we perform the embedding by feeding each string into the CLIP model to obtain the embedding vectors, finding their mean values, and then weighting and summing each embedding vector according to its cosine similarity with respect to that mean value to obtain the final embedding vector. The embedding model we use is CLIP ViT-B/32, while the embedding vector for a task is $512$D.

\noindent \textbf{4. Head-parameter training:} The projection head parameters $W_{\text{head}} = \{W_1, W_2\}$ are optimized to align ResNet-18 image features with CLIP text embeddings. Given an input image $x$ with class label $y$, we first extract frozen ResNet-18 features $E_{\text{image}}(x) \in \mathbb{R}^{512}$, then compute the projected features through two linear transformations: 
$$
    h_{\text{proj}} = \left(E_{\text{image}}(x) W_1\right) W_2 \in \mathbb{R}^{512}
$$
where $W_1 \in \mathbf{R}^{512\times16}$ and $W_2 \in \mathbf{R}^{16\times512}$ are learnable low-rank adapters. The classification logits for $k$ classes are obtained via dot-product similarity between $h_{\text{proj}}$ and CLIP text embeddings $E_{\text{text}}(c_i) \in \mathbb{R}^{512}$ of each class description $c_i$:
$$
    s_i = h_{\text{proj}} \cdot E_{\text{text}}(c_i) \quad \text{for } i=1,\dots,k
$$
The standard cross-entropy loss $\mathcal{L} = -\sum_{i=1}^k y_i \log \frac{e^{s_i}}{\sum_j e^{s_j}}$ is minimized to train $W_1$ and $W_2$, establishing task-specific decision boundaries in CLIP's multimodal space.
where $E_{\text{image}}$ denotes ResNet-18 features and $E_{\text{text}}$ CLIP text embeddings

This methodology produces 36,000 text-to-weight pairs (12,000 per base dataset), with generated weights $W_{\text{head}}\in\mathbb{R}^{512\times16\times2}$ through parameter flattening. More detailed dataset settings are provided in \autoref{app:implementation}.

\noindent \textbf{Data Splits.} The dataset is partitioned at the task level: For each of the three base datasets (CIFAR-100, Caltech256, and TinyImageNet), 80\% (9,600) of their 12,000 constructed subtasks are designated as seen tasks for training, while the remaining 20\% (2,400) are held out as unseen tasks for evaluation. Since there is no duplication of the 12,000 subtasks in each dataset, there is no overlap between the Seen Task and the Unseen Task.

\noindent \textbf{Preprocessing.} The weight preprocessing pipeline involves three key steps: (1) \textbf{Chunking} partitions flattened network parameters into 576 uniform blocks, (2) \textbf{Normalization} scales each block to the [-1, 1] range using min-max scaling, and (3) \textbf{Token Mapping} projects each normalized block to 1024-dimensional tokens via dedicated linear layers for transformer input. During generation, an inverse process applies corresponding linear transformations and denormalization to reconstruct the original parameter dimensions and scales.

\noindent \textbf{Training Setup.} The prediction network employs a mask-free transformer decoder architecture trained with batch size 32 and Adam optimizer ($lr=4\times\!10^{-4}$), using cosine learning rate scheduling with 5-epoch warmup and gradient clipping at 0.1. The adversarial discriminator—a 4-layer MLP with ReLU activations—is alternately trained every epoch at $lr=10^{-2}$. For permutation symmetry, we randomly perform a permutation symmetry transformation of the model weights in a batch with the pre-transformation batch to compute the permutation symmetry constraint loss, with the coefficient $\lambda_1=0.1$. All results are averages of three independent experiments.

\noindent \textbf{Implementation.} All experiments were conducted on 4×NVIDIA RTX 4090 GPUs (24GB VRAM) with PyTorch 2.1. Each training run completes in 1 hour of wall-clock time using mixed-precision acceleration. For reproducibility, we fix random seeds to 42,34,3407 for each independent experiment across PyTorch, NumPy, and Python's native RNG.
\vspace{-0.5cm}
\begin{figure*}[t]
\centering
\includegraphics[width=0.8\linewidth]{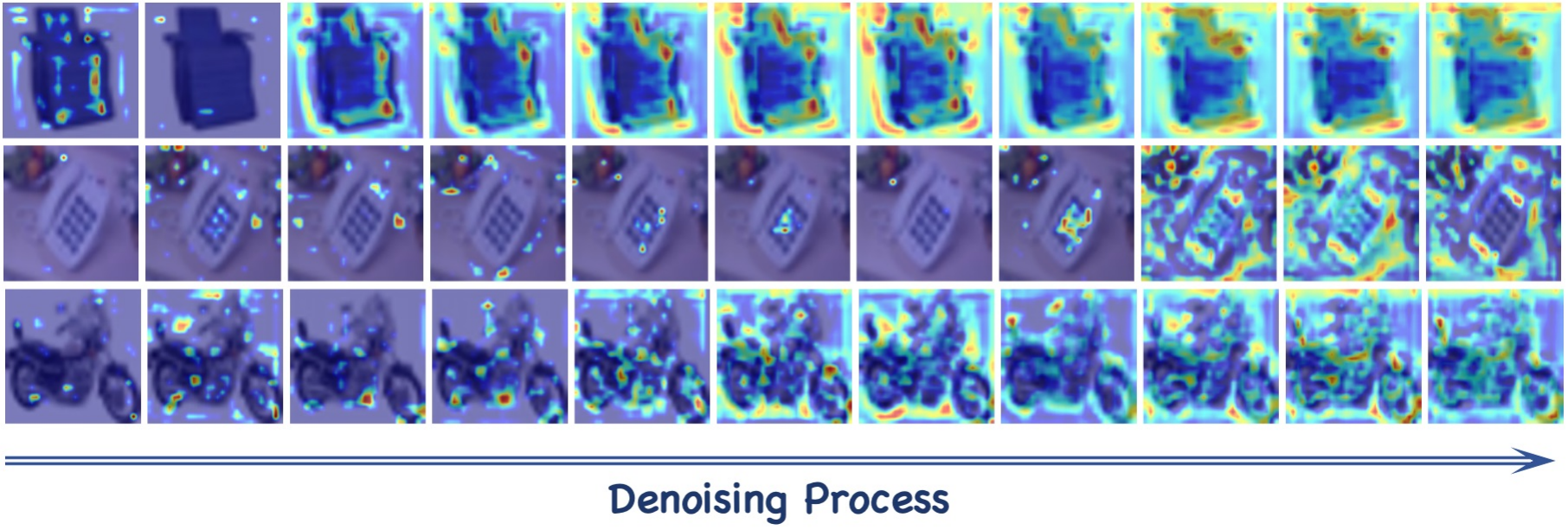}
\vspace{-8pt}
\caption{We visualize a gradCAM plot of the generated model weights on the data for the target task during the denoising process.}
\label{fig:active}
\vspace{-0.5cm}
\end{figure*}

\vspace{-8pt}
\subsection{Performance Evaluation on Seen and Unseen Tasks}
\label{subsec:evaluation}
\textbf{Baselines} We establish two key baselines to contextualize T2W's performance:

\underline{Universal Model}: Trains the CLIP classification head directly on full datasets using Adam optimizer ($\eta=3\!\times\!10^{-4}$, 200 epochs). Thanks to the compatibility of the CLIP classification head, the trained universal model can do a fair comparison with the T2W model on all downstream tasks.

Our T2W framework evaluates two variants:

\underline{T2W-NL}: Natural language task descriptions (e.g., "A small freshwater fish with vibrant orange-gold scales, commonly kept in aquariums.")

\underline{T2W-CLIP}: Standard CLIP prompts ("A photo of \{class\_i\}")

\noindent \textbf{Metrics.} We evaluate performance through two principal metrics:

\underline{Task Accuracy (Accuracy)}: Average test-set accuracy of models instantiated with generated weights across all tasks:
$$
\text{Acc} = \frac{1}{|T|}\sum_{t\in T} \left( \frac{1}{|D_t^{\text{test}}|}\sum_{(x,y)\in D_t^{\text{test}}} \mathbb{I}\left(f_{\theta_g}(x) = y\right) \right)
$$
where $T$ denotes either seen or unseen task sets, $\theta_g$ the generated weights, and $\mathbb{I}(\cdot)$ the indicator function.

\underline{Weight Similarity (Loss)}: Mean squared error between generated and target weights:
$$
\text{Loss} = \mathbb{E}_{t\sim T} \left[ \|\theta_g^{(t)} - \theta_{\text{target}}^{(t)}\|_2^2 \right]
$$
measuring parameter-space fidelity to conventionally trained weights.

\noindent \textbf{Results.} From the experimental results (\autoref{tab:performance}), it can be seen that T2W performs well on both the set of tasks that have been seen during training (\textbf{Seen Task}) and the set of tasks that have not been seen during training (\textbf{Unseen Task}), and even the average accuracy on Unseen Task is generally higher than that of Seen Task (Since the magnitude of Loss varies widely across datasets, the Avg column does not report Loss data). This fully proves the strong generalization ability that T2W possesses. At the same time, T2W achieves superior performance compared to universal models based on gradient training (\textbf{the average accuracy is generally improved by about 11\% on CIFAR-100, and about 10\% on Tiny-ImageNet as well as Caltech-256}). Also, we can observe that the average accuracy of the T2W-NL variant is generally higher than that of T2W-CLIP, which may be due to the fact that category descriptions based on long natural language may lead to a more evenly distributed textual feature space, which can help T2W to better realize the mapping from textual space to weight space.

\noindent \textbf{Visualization} In order to demonstrate the effect of T2W more intuitively, we randomly selected 10 sample points in the \textbf{Unseen Task} of each dataset, input the list of task descriptions of the sample points into the T2W model, and visualized the accuracies of the weights at different denoising steps of T2W on the corresponding test sets. As shown in \autoref{fig:denoise-vis}, it can be seen that the closer the denoising step is to the later stage, the higher the accuracy of the generated model on the test set, and eventually reaches a high level of accuracy. This can indicate that T2W approximates the optimized weights well during the denoising process. At the same time, we observe that the generated weights tend to have a significant effect improvement on the test set only after 900 steps of denoising, which may be caused by the fact that the neural network weights are an extremely sensitive data, and the effect can only be reflected on the test set when the distribution of the generated weights is very close to the target distribution.

We also used gradCAM \cite{selvaraju2017grad} to visualize the activation maps of the models obtained using generative weights on the target task dataset during the denoising process. As shown in \autoref{fig:active}, it can be seen that as the denoising steps increase, the generative weights reveal stronger and stronger attention on the target data (the activation maps become more and more focused on objects with semantic information in the image), which further demonstrates the effectiveness of the T2W denoising process.

\vspace{-8pt}
\subsection{Ablation Studies}
\label{subsec:ablation}
\vspace{-8pt}

\begin{table}[!ht]
\centering
\caption{Ablation Study of T2W Framework Components.}
\label{tab:ablation}
\vspace{-8pt}
\sisetup{exponent-product = \ensuremath{\cdot}} 
\resizebox{0.8\columnwidth}{!}{
\begin{tabular}{@{}l l S[table-format=1.2e-1] S[table-format=2.2]@{}}
\toprule
\multicolumn{1}{c}{\multirow{2}{*}{Setting}} & \multicolumn{1}{c}{\multirow{2}{*}{Dataset}} & \multicolumn{2}{c}{T2M} \\
\cmidrule(lr){3-4}
 &  & \multicolumn{1}{c}{Loss} & \multicolumn{1}{c}{Accuracy (\%)} \\
\midrule
\multirow{3}{*}{Full T2W}        & CIFAR-100      & 2.98e-5 & \textbf{65.99} \\
                                  & Tiny-ImageNet  & 1.25e-4 & \textbf{70.42} \\
                                  & Caltech-256    & 2.21e-5 & \textbf{81.89} \\
\midrule
\multirow{3}{*}{w/o Adversarial} & CIFAR-100      & 1.92e-4 & 62.96 \\
                                  & Tiny-ImageNet  & 2.01e-4 & 67.55 \\
                                  & Caltech-256    & 2.24e-5 & 81.22 \\
\midrule
\multirow{3}{*}{w/o Permutation} & CIFAR-100      & 2.98e-5 & \underline{65.52} \\
                                  & Tiny-ImageNet  & 1.29e-4 & \underline{70.29} \\
                                  & Caltech-256    & 2.44e-5 & \underline{81.32} \\
\midrule
\multirow{3}{*}{w/o Both}        & CIFAR-100      & 1.88e-4 & 63.15 \\
                                  & Tiny-ImageNet  & 2.09e-4 & 67.32 \\
                                  & Caltech-256    & 2.23e-5 & 81.29 \\
\bottomrule
\vspace{-1cm}
\end{tabular}}
\end{table}

\begin{table*}[t]
\caption{Comparative Experimental Results of Model Fusion using T2W (Accuracy).}
\label{tab:fusion}
\vspace{-8pt}
\resizebox{0.8\linewidth}{!}{
\begin{tabular}{ccccccccccc}
\toprule
\multirow{3}{*}{Model Index} & \multicolumn{10}{c}{Datasets}                                                                                                          \\
\cmidrule(lr){2-11}
                             & \multicolumn{3}{c}{CIFAR-100 S} & \multicolumn{3}{c}{Tiny-Imagenet S} & \multicolumn{3}{c}{Caltech-256 S} & \multirow{2}{*}{Avg Total} \\
                             \cmidrule(lr){2-4} \cmidrule(lr){5-7} \cmidrule(lr){8-10} 
                             & Dataset A  & Dataset B & Avg    & Dataset A   & Dataset B   & Avg     & Dataset A   & Dataset B  & Avg    &                            \\
                             \cmidrule(lr){1-11}
Model A                      & \textbf{73.29}      & 22.22     & 47.75  & \textbf{75.71}       & 8.01        & 41.86   & \textbf{80.74}       & 19.12      & 49.93  & 46.51                \\
Model B                      & 23.29      & \textbf{83.33}     & 53.31  & 12.02       & \textbf{86.44}       & 49.23   & 18.52       & \textbf{92.16}      & 55.34  & 52.62                \\
Git Re-basin\cite{ainsworth2022git} & 53.14      & 49.67     & 51.41 & 68.02       & 55.11       & \underline{61.57}   & 63.7        & 67.16      & 65.43  & \underline{59.46}                \\
T2W Generation               & \underline{72.71}      & \underline{79.67}     & \textbf{76.19}  & \underline{73.14}       & \underline{85.78}       & \textbf{79.46}   & \underline{72.59}       & \underline{88.73}      & \textbf{80.66}  & \textbf{78.77} \\
\bottomrule

\end{tabular}}
\end{table*}

\begin{figure*}[t]
\centering
\includegraphics[width=0.85\linewidth]{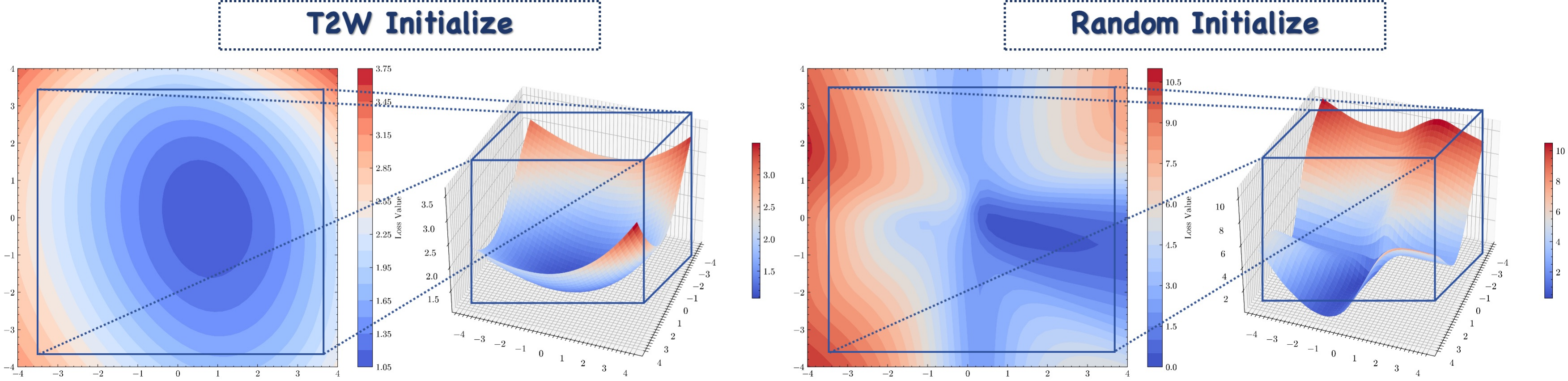}
\caption{Shape of the loss basin when training the model with different initialization methods.}
\label{fig:loss}
\vspace{-0.1cm}
\end{figure*}

\textbf{Setting.} We evaluate component contributions through four model variants: (1) Full T2W, (2) without Adversarial ($\lambda_2=0$), (3) without Permutation ($\lambda_1=0$), (4) without Both ($\lambda_2=0, \lambda_1=0$). Each variant is evaluated on all three datasets' \textbf{unseen tasks} using the Accuracy and Loss metrics. Training follows the same hyperparameters as the main experiments except for disabled components, with evaluations conducted on 500 randomly sampled tasks per dataset to ensure computational efficiency. All variants share identical initialization seeds and training schedules.

\noindent\textbf{Analysis.} The ablation study validates the critical roles of adversarial training and permutation-equivariant constraints in T2W. The results are shown in \autoref{tab:ablation}. Removing adversarial training ($\lambda_2\!=\!0$) significantly degrades performance on CIFAR-100 (accuracy drops from \textbf{65.99\%} to \textbf{62.96\%}, loss increases from $2.98\!\times\!10^{-5}$ to $1.92\!\times\!10^{-4}$) and Tiny-ImageNet (\textbf{70.42\%} $\rightarrow$ \textbf{67.55\%}, $1.25\!\times\!10^{-4}$ $\rightarrow$ $2.01\!\times\!10^{-4}$), confirming its necessity for enforcing weight-space manifold validity. Disabling permutation-equivariance ($\lambda_1\!=\!0$) causes subtle accuracy declines (e.g., \textbf{65.99\%} $\rightarrow$ \textbf{65.52\%} on CIFAR-100), highlighting its role in handling architectural symmetries. The combined removal of both components leads to the worst performance (e.g., Tiny-ImageNet accuracy: \textbf{67.32\%} vs. \textbf{70.42\%}), demonstrating their complementary effects. Notably, Caltech-256 exhibits robustness to adversarial removal (accuracy: \textbf{81.89\%} $\rightarrow$ \textbf{81.22\%}), suggesting task-dependent sensitivity. These findings emphasize that adversarial training ensures distributional validity for unseen tasks, while permutation constraints enable symmetry-aware generation, jointly enabling T2W's generalization capabilities.


\vspace{-12pt}
\subsection{T2W for Text-Driven Model Fusion}
\label{subsec:fusion}
\noindent \textbf{Setting.} For each dataset, two 8-class classification tasks (recorded as Dataset A/B.) with non-overlapping categories are constructed to train independent models (Model A/B). Baseline fusion method employ Git Re-Basin \cite{ainsworth2022git} permutation alignment, while T2W generates fused weights solely from combined text prompts. All methods were tested on two test sets (Dataset A/B).

\noindent \textbf{Analysis.}
The experimental results reveal T2W's superior performance over Git Re-Basin in model weight fusion. Git Re-Basin, which aligns neural network parameters through permutation symmetry matching and linear interpolation ($\theta_{\text{fuse}} = 0.5\theta_A + 0.5\theta_B$), achieves moderate accuracy (59.46\% average) but causes significant performance degradation on original tasks. For instance, Model A's accuracy on CIFAR-100 drops from 73.29\% to 53.14\% after fusion, while Model B decreases from 83.33\% to 49.67\%. This stems from its geometric parameter-space alignment approach that prioritizes weight-space continuity over task semantics preservation.

In contrast, T2W achieves a remarkable 78.77\% average accuracy by directly generating fused weights from combined text descriptions. The framework demonstrates unique advantages: 1) Preserving original task capabilities (72.71\% vs Model A's 73.29\% on CIFAR-100) while enabling synergistic enhancement (79.67\% vs Model B's 83.33\%); 2) Exhibiting robust cross-dataset generalization, particularly on Tiny-ImageNet, where it outperforms Git Re-Basin by 17.89\% (79.46\% vs 61.57\%). This semantic-driven fusion mechanism \textbf{translates task relationships into weight-space dynamics through natural language}, avoiding \textbf{the limitations of purely geometric fusion methods}.

\begin{table*}[t]
\caption{Comparative experimental results of weight initialization using T2W.}
\label{tab:init}
\vspace{-5pt}
\resizebox{0.8\linewidth}{!}{
\begin{tabular}{ccccccccc}
\toprule
\multirow{3}{*}{Initialize Method} & \multicolumn{8}{c}{Datasets}                                                                                                         \\
\cmidrule(lr){2-9}
                                   & \multicolumn{2}{c}{CIFAR-100 S} & \multicolumn{2}{c}{Tiny-Imagenet S} & \multicolumn{2}{c}{Caltech-256 S} & \multicolumn{2}{c}{Avg} \\
                                   \cmidrule(lr){2-3} \cmidrule(lr){4-5} \cmidrule(lr){6-7} \cmidrule(lr){8-9}
                                   & T.Loss        & Accuracy        & T.Loss          & Accuracy          & T.Loss         & Accuracy         & T.Loss    & Accuracy    \\
                                   \cmidrule(lr){1-9}
Xavier Uniform\cite{glorot2010understanding}                     & \underline{1.05}          & \underline{67.18}           & \underline{0.708}           & \underline{82.38}             & 0.369          & 87.35            & \underline{0.709}     & \underline{78.97}       \\
Xavier Normal\cite{glorot2010understanding}                      & 1.59          & 57.51           & 0.866           & 77.63             & \underline{0.344}          & \underline{88.52}            & 0.933     & 74.55       \\
Uniform                                                          & 2.63          & 39.75           & 2.57            & 44.51             & 2.94           & 9.41             & 2.713     & 31.22       \\
Normal                                                           & 1.57          & 47.25           & 1.46            & 55.87             & 2.67           & 21.17            & 1.900     & 41.43       \\
Kaiming Normal\cite{he2015delving}                               & 1.07          & 62.81           & 0.866           & 78.25             & 0.454          & 88.23            & 0.797     & 76.43       \\
Kaiming Uniform\cite{he2015delving}                              & 1.19          & 63.68           & 0.709           & 81.38             & 0.588          & 84.71            & 0.829     & 76.59       \\
T2W Initialize                                                   & \textbf{0.781}         & \textbf{74.18}           & \textbf{0.421}           & \textbf{84.12}             & \textbf{0.253}          & \textbf{92.06}            & \textbf{0.485}     & \textbf{83.45}    \\ \bottomrule  
\end{tabular}
}
\end{table*}

\begin{figure*}[h]
\centering
\includegraphics[width=0.85\linewidth]{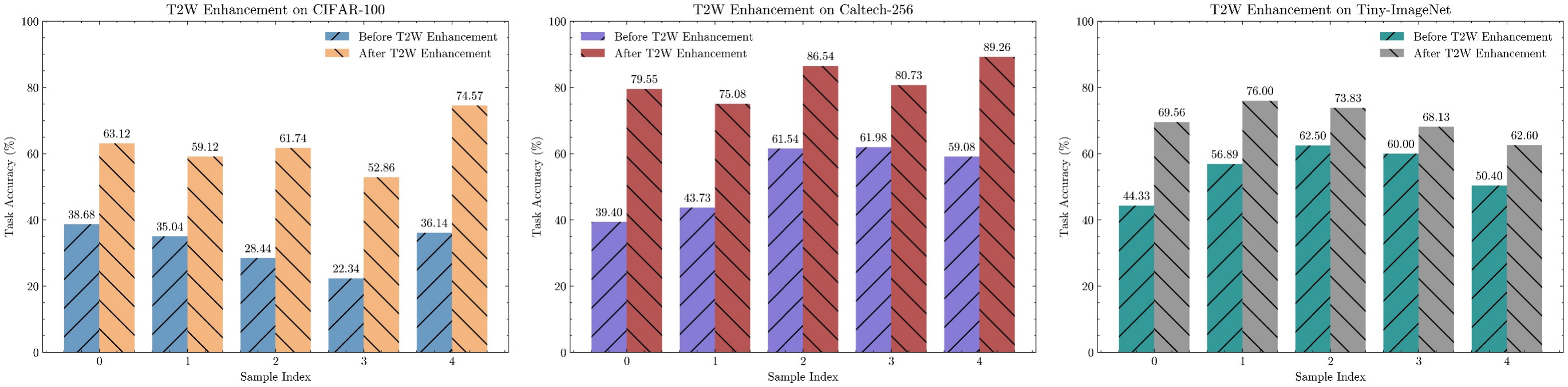}
\vspace{-8pt}
\caption{Experimental results of weight enhancement using T2W.}
\label{fig:enhance}
\vspace{-4pt}
\end{figure*}

\vspace{-12pt}
\subsection{T2W for Initializing Parameters}
\label{subsec:initialization}
\noindent \textbf{Setting.} T2W-generated weights serve as initialization for supervised training on target tasks, compared against standard methods: Xavier\cite{glorot2010understanding}/Kaiming\cite{he2015delving} (uniform/normal variants) and random initialization (uniform $\mathcal{U}(-0.1,0.1)$/normal $\mathcal{N}(0,0.01)$). All methods train for 50 epochs with identical Adam ($\eta=3\!\times\!10^{-4}$) and cosine scheduling, measuring both final test accuracy (\textbf{Accuracy}) and test loss (\textbf{T.Loss}). It is worth noting that initialization using \textbf{the T2W model requires only a textual description of the target task}.

\noindent \textbf{Analysis.} Experimental results demonstrate T2W's significant advantages over conventional initialization methods across all datasets. As shown in \autoref{tab:init}, T2W achieves state-of-the-art performance with an average accuracy of 83.45\%, outperforming the best baseline (Xavier Uniform at 78.97\%) by 4.48\%. This superiority is particularly pronounced in large-scale tasks, where T2W reaches 92.06\% accuracy on Caltech-256 - 3.54\% higher than Xavier Normal and 82.65\% better than basic Uniform initialization. The training loss metric further confirms this gap, with T2W's average loss (0.485) being 39\% lower than Xavier Uniform and 70\% lower than Normal initialization.

Traditional approaches exhibit inherent limitations due to their reliance on statistical heuristics rather than task semantics. Uniform $\mathcal{U}(-0.1,0.1)$ and Normal $\mathcal{N}(0,0.01)$ initializations perform worst (31.22\% and 41.43\% average accuracy, respectively), as their fixed magnitude distributions fail to capture the complex geometry of trainable parameter manifolds. These methods essentially scatter weights randomly across high-dimensional space, creating initialization points that require extensive optimization to reach coherent functional regions. For example, the catastrophic 9.41\% accuracy of Uniform initialization on Caltech-256 demonstrates how random magnitude selection disrupts emergent weight-space patterns essential for task performance.

T2W fundamentally redefines parameter initialization by incorporating textual task semantics into weight generation. Unlike conventional methods that rely on statistical heuristics ($\mathcal{U}$/$\mathcal{N}$ distributions) or activation scaling rules (Xavier/Kaiming), T2W directly maps natural language descriptions to optimized parameter configurations. This enables task-specific initialization in functional weight regions, bypassing traditional random exploration.

\noindent \textbf{Visualization.}
In order to see more intuitively the advantages of T2W initialization over random initialization, we plotted the \textbf{loss landscapes} of the models obtained by training with the two initialization methods, as shown in \autoref{fig:loss}. Details of the visualized loss landscape are available in \autoref{app:landscapes}.

It can be seen that the model after T2W initialization is in a relatively flat and homogeneous loss basin, while the random initialization is in an irregular hyperplane. This leads to the optimization difficulty as well as training stability of the random initialization being much lower than that of the T2W initialization, as the model is more likely to fall into local minima or saddle points in the irregular hyperplane.

\vspace{-8pt}
\subsection{T2W for Weight Enhancement}
\label{subsec:enhancement}
\noindent \textbf{Setting.} We simulate under-trained models by early stopping at 50\% training epochs on CIFAR-100/Caltech256/TinyImageNet \textbf{Subset} (25 epochs). This will cause the model to not fully converge, which will affect the test set accuracy of the model. Then apply T2W's denoising process (1000 diffusion steps) to obtain enhanced weights. In order to visualize the effect of T2W weight enhancement, we selected $5$ samples for each dataset from their \textbf{unseen tasks}, and each sample reported the test set accuracy (\textbf{Task Accuracy}) before and after enhancement. The results are shown in \autoref{fig:enhance}.

\noindent \textbf{Analysis.} As can be seen from the results, after the weight enhancement by T2W, all the samples dramatically improve the model's capability, with the largest metric increase of \textbf{36.14\% $\rightarrow$ 74.57\%} on the CIFAR-100 dataset, the largest increase of \textbf{39.40\% $\rightarrow$ 79.55\%} on Caltech-256, and the increase on Tiny-ImageNet is generally slightly lower, but the maximum is also \textbf{44.33\% $\rightarrow$ 69.56\%}. This fully demonstrates the potential of weight enhancement using T2W.

\vspace{-4pt}
\section{Most Relevant Work}
Our work extends \citet{li2024text}'s architecture with flexible classification headers for varied class numbers; proposes three Text2Weight downstream tasks; and provides comprehensive experimental analysis with multiple visualizations.

\vspace{-4pt}
\section{Conclusion}
We propose \textbf{T2W}, a diffusion transformer framework that synthesizes neural network weights directly from natural language task descriptions. By integrating hierarchical weight-space processing, text-conditioned diffusion modeling, and symmetry-aware adversarial training, T2W achieves robust generalization to unseen tasks. The framework demonstrates its versatility through novel applications—including task-specific weight initialization, post-training weight enhancement, and text-guided model fusion, highlighting the potential of treating weight-space as a manipulable modality.
\clearpage

\section{Acknowledgement}
This work was supported by Guangzhou-HKUST(GZ) Joint Funding Program(Grant No.2023A03J0008), Education Bureau of Guangzhou Municipality.

\bibliographystyle{ACM-Reference-Format}
\balance
\bibliography{sample-sigconf}

\newpage
\onecolumn
\appendix

\section{Extended Theoretical Analysis of Error Terms for Implicit Augmentation vs. Explicit Constraints}
\label{app:theoretical}

In order to make generative models generate better weights, a lot of work has focused on the property that the weights should be permutation invariant\cite{soro2024diffusion, li2024text, peebles2022learning}. Naturally, when training a generative model, the performance of the weights generated by the generative model can be enhanced by randomly augmenting the data points in the weight dataset with a random permutation (by directly rearranging the weights of a multi-layer neural network to obtain a new network with the same functionality but different permutations of weights). Our T2W, however, directly adds alignment symmetry constraint loss during Diffusion training, and we will demonstrate the superiority of directly adding explicit constraints compared to augmenting the weighted data and give the source and analysis of the error of implicit augmentation in the following.

\subsection{Problem Formulation and Training Objectives}

\textbf{Definitions:}

\begin{itemize}
\item \textit{Symmetry Group}: Let $\mathcal{G}$ denote a finite permutation group with $|\mathcal{G}|=M$, where each $g\in\mathcal{G}$ corresponds to a symmetry transformation (e.g., neuron permutations in a neural network layer).

\item \textit{Noise Prediction Model}: $\epsilon_\phi(\cdot)$ is a neural network that predicts the noise added during the diffusion process. Its equivariance error is defined as:
\begin{equation}
\eta(g,\theta_n) = \epsilon_\phi(g\cdot\theta_n) - g\cdot\epsilon_\phi(\theta_n),
\end{equation}
where $\theta_n$ represents the perturbed weights at diffusion step $n$.
\end{itemize}

\textbf{Training Objectives:}

\begin{itemize}
\item \textit{Explicit Symmetry Loss ($L_{\text{sym}}$)}:
\begin{equation}
L_{\text{sym}} = \mathbb{E}_{g\sim\mathcal{G},\theta_n}\left[\|\eta(g,\theta_n)\|^2\right]
\end{equation}
This directly penalizes deviations from equivariance by minimizing the squared norm of the equivariance error.

\item \textit{Implicit Augmentation Loss ($L_{\text{aug}}$)}:
\begin{equation}
L_{\text{aug}} = \mathbb{E}_{g\sim\mathcal{G},\theta_n}\left[\|\epsilon - \epsilon_\phi(g\cdot\theta_n,n,v_c)\|^2\right],
\end{equation}
where $\epsilon$ is the ground-truth noise. This indirectly encourages equivariance by training on symmetry-augmented data but does not explicitly enforce it.
\end{itemize}

\subsection{Key Assumptions}

\begin{itemize}
\item \textit{Assumption 1 (Finite Symmetry Group)}: 
$\mathcal{G}$ is a finite group with $|\mathcal{G}|=M$. This ensures all symmetry operations are discrete and enumerable, critical for analyzing coverage during training.

\item \textit{Assumption 2 (Model Capacity)}: 
The model $\epsilon_\phi$ has sufficient capacity to approximate equivariance but may not perfectly satisfy it. Thus, $\eta(g,\theta_n) \neq 0$ in general.

\item \textit{Assumption 3 (Lipschitz Continuity)}: 
$\epsilon_\phi$ is $L$-Lipschitz continuous with respect to weight transformations:
\begin{equation}
\|\epsilon_\phi(g\cdot\theta_n) - \epsilon_\phi(\theta_n)\| \leq L\|g\cdot\theta_n - \theta_n\|.
\end{equation}
This bounds how much the noise prediction can vary under symmetry transformations.
\end{itemize}

\subsection{Error Decomposition for Implicit Augmentation}

\begin{theorem}[Loss Decomposition]
The implicit augmentation loss $L_{\text{aug}}$ decomposes into three terms:
\begin{equation}
L_{\text{aug}} = \underbrace{\mathbb{E}_{g,\theta_n}\left[\|\epsilon - g\cdot\epsilon_\phi(\theta_n)\|^2\right]}_{\text{Ideal Loss}} + \underbrace{\mathbb{E}_{g,\theta_n}\left[\|\eta(g,\theta_n)\|^2\right]}_{\text{Equivariance Error}} - \underbrace{2\mathbb{E}_{g,\theta_n}\left[\langle \epsilon - g\cdot\epsilon_\phi(\theta_n), \eta(g,\theta_n) \rangle\right]}_{\text{Cross-Term}}.
\end{equation}
\end{theorem}

\begin{proof}
Expand $L_{\text{aug}}$ using $\epsilon_\phi(g\cdot\theta_n) = g\cdot\epsilon_\phi(\theta_n) + \eta(g,\theta_n)$:
\begin{align}
L_{\text{aug}} &= \mathbb{E}\left[\|\epsilon - (g\cdot\epsilon_\phi(\theta_n) + \eta(g,\theta_n))\|^2\right] \\
&= \mathbb{E}\left[\|(\epsilon - g\cdot\epsilon_\phi(\theta_n)) - \eta(g,\theta_n)\|^2\right] \\
&= \mathbb{E}\left[\|\epsilon - g\cdot\epsilon_\phi(\theta_n)\|^2\right] + \mathbb{E}\left[\|\eta(g,\theta_n)\|^2\right] - 2\mathbb{E}\left[\langle \epsilon - g\cdot\epsilon_\phi(\theta_n), \eta(g,\theta_n) \rangle\right].
\end{align}
The cross-term arises from the interaction between the ideal noise residual and the equivariance error.
\end{proof}

\subsection{Gap Between Implicit and Explicit Objectives}

\begin{lemma}[Error Gap Bound]
The difference between implicit and explicit losses is bounded by:
\begin{equation}
|L_{\text{aug}} - L_{\text{sym}}| \leq \delta^2 - 2\delta\gamma,
\end{equation}
where:
\begin{itemize}
\item $\delta = \max_{g\in\mathcal{G}} \|\epsilon - g\cdot\epsilon_\phi(\theta_n)\|$ (deviation from ideal equivariant noise prediction),
\item $\gamma = \max_{g\in\mathcal{G}} \|\eta(g,\theta_n)\|$ (equivariance error magnitude).
\end{itemize}
\end{lemma}

\begin{proof}
Using the Cauchy-Schwarz inequality:
\begin{equation}
\langle \epsilon - g\cdot\epsilon_\phi(\theta_n), \eta(g,\theta_n) \rangle \leq \|\epsilon - g\cdot\epsilon_\phi(\theta_n)\| \cdot \|\eta(g,\theta_n)\| \leq \delta\gamma.
\end{equation}
Substitute into Theorem 1:
\begin{align}
|L_{\text{aug}} - L_{\text{sym}}| &\leq \mathbb{E}\left[\|\epsilon - g\cdot\epsilon_\phi(\theta_n)\|^2\right] - 2\mathbb{E}[\delta\gamma] \\
&\leq \delta^2 - 2\delta\gamma.
\end{align}
At this point, we derive an upper bound on the error between implicit data augmentation and explicit permutation symmetry constraints.
\end{proof}

By organizing the above equation we get $|L_{\text{aug}} - L_{\text{sym}}|\leq \delta(\delta-2\gamma)$, because of $\eta(g,\theta_n) = \epsilon_\phi(g\cdot\theta_n) - g\cdot\epsilon_\phi(\theta_n)$, also $\delta(\delta-2\gamma)$ is a non-negative value. When condition $\delta-2\gamma=0$ is difficult to meet, only $\delta=0$ can guarantee that the error between the two losses is $0$. Obviously, since random noise tends not to have \textbf{permutation symmetry}, it is difficult for $\delta$ to be trained to converge to $0$. This also ultimately leads to the fact that \textbf{implicit data enhancement methods often do not serve the purpose of permutation symmetry constraints well}, which is why we introduce explicit permutation symmetry constraints in the T2W method to enhance the quality of the generated weights from the generative model.

\section{Introduction to the Dataset}
\label{app:implementation}
In this section, we describe the dataset construction process in detail and give some details of the algorithmic process and tables. Helps to understand our dataset construction in depth.

\subsection{Introduction to Base Dataset}
\label{app:basedatasets}

\noindent \textbf{CIFAR-100} The CIFAR-100 dataset, introduced by \cite{krizhevsky2009learning}, serves as an expanded version of CIFAR-10 containing 100 fine-grained object categories organized into 20 superclasses. Each class contains 600 $32\times32$ RGB images (500 training + 100 test) capturing challenging visual concepts ranging from aquatic mammals to household items, with significant intra-class variation and subtle inter-class differences that make it particularly suitable for testing model generalization capabilities.

\noindent \textbf{Caltech-256} The Caltech-256 dataset, developed by \cite{griffin2007caltech}, extends the original Caltech-101 with expanded category coverage to 256 object classes, each containing at least 80 high-resolution images averaging $350\times300$ pixels. This dataset emphasizes real-world recognition challenges through complex background scenes, diverse object scales, and varied viewpoints, making it particularly valuable for studying fine-grained classification and partial occlusion handling.

\noindent \textbf{TinyImageNet} The TinyImageNet dataset, curated by \cite{tiny-imagenet}, provides a computationally tractable subset of ImageNet containing 200 classes with reduced image resolution ($64\times64$ pixels). Each class includes 500 training images alongside 50 validation and 50 test images, preserving essential visual features while enabling efficient experimentation with model architectures and training protocols under resource constraints.

\subsection{Text-to-Feature Vector Generation Protocol} 
\label{app:text2feature}
Given a class set \( S = \{s_1, s_2, \dots, s_k\} \) (e.g., \texttt{["otter", "lamp", ..., "flatfish"]}), the workflow for generating fused text feature vectors is structured as follows:

\begin{algorithm}[H]
\caption{Text Feature Vector Generation and Fusion}
\begin{algorithmic}[1]
\Require Class set \( S \), CLIP text encoder \( E_{\text{text}} \), fusion function \( \texttt{fuse\_features}(\cdot) \)
\Ensure Fused text feature vector \( v_c \in \mathbb{R}^{512} \)

\State \textbf{Step 1: Task Description Generation}
\For{each class \( s_i \in S \)}
    \State \( c_i \gets \text{``A photo of ''} + s_i \) \Comment{Template-based text instantiation}
\EndFor

\State \textbf{Step 2: CLIP Embedding Extraction}
\For{each description \( c_i \)}
    \State \( e_i \gets E_{\text{text}}(c_i) \) \Comment{CLIP encodes \( c_i \) to \( e_i \in \mathbb{R}^{512} \)}
\EndFor
\State \( E \gets [e_1; e_2; \dots; e_k] \in \mathbb{R}^{k \times 512} \) \Comment{Aggregate embeddings}

\State \textbf{Step 3: Feature Fusion via \texttt{fuse\_features}}
\State \( v_c \gets \texttt{fuse\_features}(E) \), executed as:
    \begin{enumerate}
    \item Compute global mean vector: 
        \[
        \mu = \frac{1}{k} \sum_{i=1}^k e_i
        \]
    \item Normalize embeddings and mean vector:
        \[
        \hat{e}_i = \frac{e_i}{\|e_i\|_2}, \quad \hat{\mu} = \frac{\mu}{\|\mu\|_2}
        \]
    \item Calculate cosine similarity weights:
        \[
        w_i = \frac{\exp\left(\hat{e}_i \cdot \hat{\mu}^\top\right)}{\sum_{j=1}^k \exp\left(\hat{e}_j \cdot \hat{\mu}^\top\right)}
        \]
    \item Generate fused vector:
        \[
        v_c = \sum_{i=1}^k w_i \cdot e_i
        \]
    \end{enumerate}
\end{algorithmic}
\end{algorithm}

\paragraph{Illustrative Class Set Example}
Table~\ref{tab:class_set} demonstrates the mapping from raw class names to templated text descriptions for the given example set \( S \).

\begin{table}[H]
\centering
\caption{Class Set to Text Description Mapping}
\label{tab:class_set}
\begin{tabular}{ll}
\toprule
Raw Class Name & Templated Text Description \\
\midrule
otter & ``A photo of otter'' \\
lamp & ``A photo of lamp'' \\
cattle & ``A photo of cattle'' \\
elephant & ``A photo of elephant'' \\
worm & ``A photo of worm'' \\
palm\_tree & ``A photo of palm tree'' \\
rocket & ``A photo of rocket'' \\
house & ``A photo of house'' \\
streetcar & ``A photo of streetcar'' \\
crab & ``A photo of crab'' \\
whale & ``A photo of whale'' \\
crocodile & ``A photo of crocodile'' \\
possum & ``A photo of possum'' \\
beaver & ``A photo of beaver'' \\
wolf & ``A photo of wolf'' \\
flatfish & ``A photo of flatfish'' \\
\bottomrule
\end{tabular}
\end{table}

\paragraph{Implementation Notes}
\begin{itemize}
\item The CLIP text encoder \( E_{\text{text}} \) corresponds to the \texttt{ViT-B/32} variant.
\item Weighted fusion prioritizes classes whose embeddings align with the global semantic centroid, suppressing outliers.
\item The final \( v_c \) serves as the conditional input to T2W's diffusion transformer for weight generation.
\end{itemize}

\begin{table}[h]
  \caption{Detiled Dataset Statistics}
  \label{tab:dataset_stats}
  \centering
  \begin{tabular}{lccc}
    \toprule
    Metric & CIFAR-100 & Caltech256 & TinyImageNet \\
    \midrule
    Original Classes & 100 & 256 & 200 \\
    Avg. Subtask Classes & 18.7 & 19.2 & 20.1 \\
    Text Embedding Dim & \multicolumn{3}{c}{512} \\
    Head Parameters & \multicolumn{3}{c}{16,384 (512×16×2)} \\
    Training Epochs & \multicolumn{3}{c}{200} \\
    Feature Dim & \multicolumn{3}{c}{512} \\
    \bottomrule
  \end{tabular}
\end{table}

\begin{lstlisting}[language=Python, caption={CLIP Adapter Implementation}, label=code:clip_adapter]
class CLIPAdapter(torch.nn.Module):
    """Adapted CLIP classifier with frozen backbones and trainable projection head"""
    
    def __init__(self, clip_model, hidden_dim=8):
        super().__init__()
        
        # Freeze CLIP model parameters
        self.clip = clip_model
        self.clip.requires_grad_(False)  # Disable gradient computation
        
        # Initialize frozen ResNet-18 backbone
        self.resnet = torchvision.models.resnet18(pretrained=True)
        for name, param in self.resnet.named_parameters():
            param.requires_grad = False  # Freeze ResNet parameters
            
        # Replace final fully-connected layer
        self.resnet.fc = torch.nn.Sequential(
            torch.nn.Linear(512, hidden_dim),  # Dimension reduction
            torch.nn.GELU(),                   # GELU activation
            torch.nn.Linear(hidden_dim, 512)   # Dimension restoration
        )

    def encode_text(self, text_inputs):
        """Extract L2-normalized text features using CLIP encoder"""
        with torch.no_grad():
            # CLIP text encoding process
            init_text_features = self.clip.encode_text(text_inputs)
            init_text_features /= init_text_features.norm(dim=-1, keepdim=True)
            text_features = init_text_features.detach().float()  # Convert to float32
        return text_features

    def encode_image(self, image_inputs):
        """Extract L2-normalized image features using CLIP encoder"""
        with torch.no_grad():
            # CLIP image encoding process
            image_features = self.clip.encode_image(image_inputs)
            image_features /= image_features.norm(dim=-1, keepdim=True)
            image_features = image_features.detach().float()  # Convert to float32
        return image_features

    def forward(self, images):
        """Forward pass through adapted ResNet backbone"""
        # Feature extraction and normalization
        adapted_features = self.resnet(images)
        adapted_features = adapted_features / adapted_features.norm(dim=-1, keepdim=True)
        
        return {
            "adapted": adapted_features,  # Return dictionary for compatibility
        }
\end{lstlisting}

\subsection{Two-Stage Adapter Training Protocol} \label{sec:adapter_training}

To ensure the generated weights reside within the target distribution manifold while maintaining task adaptability, we implement a two-stage initialization strategy for the CLIP adapter:

\begin{algorithm}[H]
\caption{Base Model Initialization \& Sub-Dataset Adaptation}
\begin{algorithmic}[1]
\Require Full dataset $\mathcal{D}_{\text{full}}$ (e.g., CIFAR-100), sub-datasets $\{\mathcal{D}_{\text{sub}}^i\}_{i=1}^N$
\Ensure Task-specific adapted models $\{f_{\theta_i}\}_{i=1}^N$

\State \textbf{Stage 1: Base Model Pretraining}
\State \hspace{0.5em} 1. Initialize CLIP adapter $f_{\theta_0}$ with random weights
\State \hspace{0.5em} 2. Train on $\mathcal{D}_{\text{full}}$ for 1 epoch:
\begin{equation*}
    \theta_0^* = \arg\min_{\theta} \mathbb{E}_{(x,y)\sim\mathcal{D}_{\text{full}}} \mathcal{L}_{\text{CE}}(f_\theta(x), y)
\end{equation*}
where $\mathcal{L}_{\text{CE}}$ denotes cross-entropy loss
\State \hspace{0.5em} 3. Freeze projection head: $\theta_0^*.\texttt{fc[-1].requires\_grad} \gets \text{False}$

\State \textbf{Stage 2: Sub-Dataset Specialization}
\For{each sub-dataset $\mathcal{D}_{\text{sub}}^i$}
    \State \hspace{0.5em} 1. Initialize adapter from base model: $f_{\theta_i} \gets f_{\theta_0^*}$
    \State \hspace{0.5em} 2. Unfreeze final layer: $f_{\theta_i}.\texttt{fc[-1].requires\_grad} \gets \text{True}$
    \State \hspace{0.5em} 3. Fine-tune with learning rate $\alpha_{\text{sub}}$:
    \begin{equation*}
        \theta_i^* = \arg\min_{\theta} \mathbb{E}_{(x,y)\sim\mathcal{D}_{\text{sub}}^i} \mathcal{L}_{\text{CE}}(f_\theta(x), y)
    \end{equation*}
    Target model weights are obtained after $32$ epochs of training.
    \State \hspace{0.5em} 4. Store weights: $\Theta_{\text{gen}} \gets \Theta_{\text{gen}} \cup \{\theta_i^*\}$
\EndFor
\end{algorithmic} 
\end{algorithm}

\section{Appendix: Loss Landscape Visualization Methodology}
\label{app:landscapes}

This appendix details the methodology for generating the loss landscape visualizations presented in Section 4.5. Our approach adapts established techniques from \cite{li2018visualizing} with task-specific modifications for neural weight space analysis.

\subsection{Experimental Setup}
\begin{itemize}
    \item \textbf{Initialization Methods}:
    \begin{itemize}
        \item T2W Initialization: Weights generated from text description \textit{"Initialize a high-performance classifier for [Task A]"}
        \item Random Initialization: Xavier normal initialization
    \end{itemize}
    
    \item \textbf{Task Selection}:
    \begin{itemize}
        \item Target Task (A): 8-class classification from CIFAR-100
        \item Unrelated Task (B): 12-class classification from Tiny-ImageNet
    \end{itemize}
\end{itemize}

\subsection{Direction Sampling}
We compute perturbation directions using a hybrid approach:

\begin{equation}
    \delta_1 = \frac{\theta_{\text{final}} - \theta_{\text{init}}}{||\theta_{\text{final}} - \theta_{\text{init}}||_2} \quad \text{(Optimization trajectory direction)}
\end{equation}

\begin{equation}
    \delta_2 \sim \mathcal{N}(0,I) \quad \text{(Random Gaussian direction)}
\end{equation}

Parameters are normalized layer-wise before direction computation to handle scale variations across network layers.

\subsection{Perturbation \& Computation}
For each initialization type and task combination:

\begin{enumerate}
    \item Train model for 50 epochs with Adam ($\eta=3\times10^{-4}$)
    \item Compute baseline loss $\mathcal{L}_0$ at converged weights $\theta^*$
    \item Generate 2D grid with perturbation coefficients $\alpha,\beta \in [-4,4]$:
    
    \begin{equation}
        \mathcal{L}(\alpha,\beta) = \mathcal{L}\left(\theta^* + \alpha\delta_1 + \beta\delta_2\right)
    \end{equation}
    
    \item Evaluate at $50\times50$ grid resolution with batch-normalized losses
\end{enumerate}

This methodology enables direct comparison of loss basin geometries while maintaining task-specific characteristics. The complete visualization code will be released with the final version of this paper. The complete algorithm flow is shown in \autoref{alg:loss_flow}.

\begin{algorithm}
\caption{2D Loss Landscape Visualization Workflow}
\label{alg:loss_flow}
\begin{algorithmic}[1]
\State \textbf{Initialize:}
\State \quad $\theta_{\text{current}} \gets \theta_0$ \Comment{Restore initial parameters}
\State \textbf{Process Directions:}
\State \quad Flatten $\delta_1$, $\delta_2$ to vectors
\State \quad Normalize $\delta_2$ using $\delta_1$'s L2-norm \Comment{Critical scaling step}
\State \quad Reconstruct original tensor shapes
\State \textbf{Create Parameter Grid:}
\State \quad Generate $\alpha$ values in $[\alpha_{\min}, \alpha_{\max}]$
\State \quad Generate $\beta$ values in $[\beta_{\min}, \beta_{\max}]$
\State \quad Initialize loss matrix $Z$
\If{Cached data exists}
    \State Load precomputed $Z$ matrix \Comment{Cache reuse}
\Else
    \For{each $\alpha_i$ in $\alpha$ grid}
        \For{each $\beta_j$ in $\beta$ grid}
            \State $\theta_{\text{perturbed}} \gets \theta_0 + \alpha_i\delta_1 + \beta_j\delta_2$
            \State Update model parameters with $\theta_{\text{perturbed}}$
            \State Compute loss over entire dataset
            \State Store result in $Z[i,j]$
        \EndFor
    \EndFor
    \State Save $Z$ matrix \Comment{Cache storage}
\EndIf
\State \textbf{Visualization:}
\If{3D plot selected}
    \State Generate 3D surface plot $(X, Y, Z^T)$
    \State Set camera view angles
\ElsIf{2D plot selected}
    \State Generate contour plot
\EndIf
\State Save figure to disk
\State \textbf{Cleanup:}
\State Restore $\theta_0$ in model \Comment{Parameter safety}
\end{algorithmic}
\end{algorithm}

\clearpage
\section{Detailed GradCAM Visualization Methodology}
\label{sec:gradcam-details}

\subsection{Denoising Step Sampling}
We analyze the denoising process by sampling model weights at intervals of 100 steps from step 900 to 1000:

\begin{equation}
\{\theta_t\}_{t=900}^{1000} \text{ where } t \in \{900, 910, \ldots, 1000\}
\end{equation}

\subsection{Visualization Pipeline}
The GradCAM implementation consists of three core components:

\begin{algorithm}[H]
\caption{GradCAM Visualization for Denoised Weights}
\begin{algorithmic}[1]
\State \textbf{Input:} Denoised weights \(\theta_t\), CIFAR100 subset \(\mathcal{D}\), sample index \(i\)
\State \textbf{Output:} Class activation map

\Function{CamVis}{$\theta_t, \mathcal{D}, i$}
    \State Initialize CLIP adapter with \(\theta_t\)
    \State Prepare text embeddings for all classes:
    \For{each class \(c \in \mathcal{C}\)}
        \State \(t_c \leftarrow \texttt{clip.tokenize(``a photo of a \{c\}'')}\) \label{line:tokenize}
    \EndFor
    \State \(T \leftarrow \texttt{torch.cat}([t_c])\) \label{line:concat}
    \State \(F_{\text{text}} \leftarrow \texttt{clip\_adapter.encode\_text}(T)\) \label{line:encode}

    \State Configure visualization targets:
    \State \(\ell_{\text{target}} \leftarrow \texttt{layer2[-1].conv2}\) \Comment{Shallow convolutional layer}
    \State \(\texttt{cam} \leftarrow \texttt{GradCAM(model, [}\ell_{\text{target}}\texttt{])}\) \label{line:gradcam}

    \State Process sample image:
    \State \(x, y \leftarrow \mathcal{D}[i]\)
    \State \(x_{\text{norm}} \leftarrow \texttt{denormalize}(x)\) \label{line:denorm}

    \State Generate activation map: \label{line:genmap}
    \State \texttt{with TemporaryGradEnable([}$\ell_{\text{target}}$\texttt{])}:
        \State $\texttt{targets} \leftarrow [\texttt{ClassifierOutputTarget}(y)]$
        \State $A \leftarrow \texttt{cam}(x_{\text{norm}}, \texttt{targets})$

    \State Return overlay visualization:
    \State \Return \(\texttt{show\_cam\_on\_image}(x_{\text{norm}}, A)\) \label{line:render}
\EndFunction
\end{algorithmic}
\end{algorithm}

\subsection{Key Implementation Details}
1. \textbf{Layer Selection}: Focus on the second residual block's final convolutional layer (ResNet's \texttt{layer2[-1].conv2}) to capture mid-level features.

\noindent 2. \textbf{Gradient Preservation}: Enable gradients only for target layer through context management:
\begin{equation}
\frac{\partial y^c}{\partial F^{(l)}_{k}(i,j)} = \underbrace{\texttt{TemporaryGradEnable}(\ell_{\text{target}})}_{\text{selective gradient flow}}
\end{equation}

\noindent 3. \textbf{Denormalization}: Recover original RGB values using dataset statistics:
\begin{equation}
x_{\text{denorm}} = x \otimes \sigma \oplus \mu,\ \sigma= \begin{bmatrix}0.229\\0.224\\0.225\end{bmatrix}, \mu= \begin{bmatrix}0.485\\0.456\\0.406\end{bmatrix}
\end{equation}

\noindent 4. \textbf{Class Targeting}: Compute gradients relative to ground-truth class through \texttt{ClassifierOutputTarget} wrapper.

\end{document}